\documentclass[letterpaper, 10 pt, conference]{ieeeconf}  

\IEEEoverridecommandlockouts                              

\usepackage{graphics} 
\usepackage{epsfig} 
\usepackage{mathptmx} 
\usepackage{times} 
\usepackage{amsmath} 
\usepackage{amssymb}  

\usepackage{url}

\usepackage{color}
\usepackage{cite}

\usepackage{calrsfs}
\DeclareMathAlphabet{\pazocal}{OMS}{zplm}{m}{n}

\UseRawInputEncoding
\newtheorem{theorem}{Theorem}[section]

\newtheorem{prop}{Proposition}[section]
\newtheorem{assumption}{Assumption}[section]

\title{\LARGE \bf
Distributionally Robust and Safe Imitation Learning*
}

\author{Ahmed Aboudonia$^{1}$ and Naira Hovakimyan$^{2}$
\thanks{*This work was supported by the Swiss National Science Foundation under grant agreement P500PT$\_$230223, the Air Force Office of Scientific Research Grant (AFOSR) Grant AF FA9550-25-1-0274, the National Aeronautics and Space Administration (NASA) under Grant 80NSSC22M0070, and the National Science Foundation (NSF) under Grants CMMI 2135925, CPS 2311085 and IIS 2331878.}
\thanks{$^{1}$Ahmed Aboudonia is with the University of California Berkeley, Berkeley, CA 94720, USA {\tt\small aboudonia@berkeley.edu}}%
\thanks{$^{2}$Naira Hovakimyan is with the University of Illinois Urbana-Champaign, Urbana, IL 61801, USA {\tt\small nhovakim@illinois.edu}}%
}

\begin{document}

\maketitle
\thispagestyle{empty}
\pagestyle{empty}

\begin{abstract}

Imitation learning (IL) has achieved remarkable success in complex decision-making tasks. 
However, its performance is highly sensitive to distribution shifts, which can pose significant safety risks. 
We propose a distributionally robust and safe IL framework that explicitly addresses both policy-induced and uncertainty-induced distribution shifts.
Our approach develops a unified framework leveraging Taylor Series Imitation Learning (TaSIL) to mitigate policy-induced shifts and distributionally robust adaptive control to handle uncertainty-induced shifts.
This architecture enables the formulation of an IL problem that optimizes performance under distributional uncertainty while systematically accounting for safety constraints. We demonstrate the effectiveness of the proposed approach on an unmanned aerial vehicle (UAV) case study where the UAV performs a task in an uncertain environment while avoiding unsafe regions.

\end{abstract}

\section{INTRODUCTION}

Learning-from-demonstrations has emerged as a compelling paradigm for enabling autonomous systems to acquire control policies from expert trajectories~\cite{ravichandar2020recent}. 
In this context, imitation learning (IL) has shown significant success in learning policies that reproduce expert behavior~\cite{zare2024survey}.
IL has been successfully applied in domains such as robotic manipulation~\cite{fang2019survey}, autonomous driving~\cite{le2022survey}, and legged robots~\cite{mirza2025imitation}.

Behavioral cloning (BC) is considered the basic IL approach, in which expert demonstrations are collected as state-action pairs~\cite{zare2024survey}.
A supervised learning problem is then solved to learn a mapping from states to actions.
To reduce the need for expert action labeling, imitation from observation has emerged as a promising approach that enables learning from state-only demonstrations \cite{torabi2019recent}.

While IL has shown promising results, it is limited by the intrinsic difficulty of distribution shift. 
One source of this shift is policy-induced distribution shift, where small deviations in the learned policy can accumulate over time, potentially causing cascading errors \cite{zare2024survey,belkhale2023data}. 
Another source is uncertainty-induced distribution shift, in which minor variations in the environment, modeled as epistemic and aleatoric uncertainties, can similarly propagate through the system, leading to degraded performance \cite{tagliabue2024efficient,gahlawat2026distributionally}.
As a result, both policy-induced and uncertainty-induced distribution shifts can propagate errors that pose significant safety risks.

To mitigate the effects of policy-induced distribution shifts, several IL variants have been developed. 
For example, approaches such as DAgger (Dataset Aggregation) \cite{ross2011reduction} and DART (Disturbances for Augmenting robot trajectories) \cite{laskey2017dart} build upon BC and use supervised learning. 
Other methods, such as inverse reinforcement learning \cite{arora2021survey} and generative adversarial imitation learning \cite{ho2016generative}, follow different principles and have also demonstrated robustness in mitigating distribution shifts \cite{zare2024survey}. 
More recently, Taylor series imitation learning (TaSIL) has been proposed to align expert and imitated policies along with their higher-order derivatives \cite{pfrommer2022tasil}.

To address uncertainty-induced distribution shifts, several approaches have been proposed. 
Robust imitation learning methods, for instance, have been developed to learn from multiple experts acting under diverse environment dynamics \cite{chae2022robust}. 
Sampling augmentation methods leverage tube-based robust MPC schemes within imitation learning frameworks such as BC and DAgger \cite{tagliabue2024efficient}. 
Distributionally robust imitation learning approaches build upon BC and employ distributionally robust optimization within a Robust Markov Decision Process framework \cite{panaganti2023distributionally}.
Recently, a layered control architecture, termed DRIP, has been proposed \cite{gahlawat2026distributionally}.
It integrates TaSIL to address the policy-induced imitation gap and $\pazocal{L}_1$-Distributionally Robust Adaptive Control ($\pazocal{L}_1$-DRAC) to mitigate the uncertainty-induced imitation gap \cite{gahlawat2025mathcal}. 

In this paper, we propose a distributionally robust and safe IL framework that adopts a similar architecture to that of DRIP, where policy-induced and uncertainty-induced distribution shifts are handled using TaSIL and $\pazocal{L}_1$-DRAC, respectively.
Unlike the standard empirical risk minimization over expert trajectories typically used in TaSIL, we reformulate the TaSIL problem to explicitly account for distributional robustness and safety requirements during training, inspired by the robustness certificates provided by $\pazocal{L}_1$-DRAC. 

To this end, we first generalize the TaSIL objective to a distributionally robust formulation that optimizes performance over an ambiguity set of state distributions, thereby accounting explicitly for distribution shifts.
Additionally, we incorporate a safety term into the TaSIL objective, ensuring that the resulting policy accounts for safety requirements.
The resulting framework is designed to be robust to both policy-induced and uncertainty-induced distribution shifts, optimize performance under distributional uncertainty, and account for safety considerations.
Finally, we demonstrate the efficacy of the proposed framework through extensive simulations on an unmanned aerial vehicle (UAV) operating in uncertain environments while avoiding unsafe regions.

The remainder of this paper is organized as follows.
Section II presents the problem setup and provides a brief overview of IL, with a focus on TaSIL.
Section III introduces our proposed distributionally robust and safe IL framework. Specifically, we discuss $\pazocal{L}_1$-DRAC and demonstrate how it can be leveraged to reformulate a distributionally robust TaSIL problem, as well as to account for safety requirements.
Section IV evaluates the performance of the proposed framework through extensive simulations on a UAV case study. Concluding remarks are provided in Section V.


\section{Problem Setup}

We consider the uncertain dynamical system whose dynamics are given by the It$\text{\^o}$ process:
\begin{equation}
    \label{eq:ito}
    d\bar{X}_{t}
    = \left[A_\mu \bar{X}_{t} + B_c\left(\bar{U}_{t} + H_\mu(\bar{X}_{t}) \right) \right]dt
    + \left[A_\sigma  + B_c H_\sigma(\bar{X}_{t})  \right]dW_{t}, 
\end{equation}
where $\bar{X}_{t} \in \mathbb{R}^n$ and $\bar{U}_{t} \in \mathbb{R}^m$ are the system state and control input vectors, respectively, and $W_{t} \in \mathbb{R}^{n_w}$ is a Brownian motion defined on the complete probability space $\left(\Omega,\pazocal{F},\mathbb{P}\right)$.
The matrices $A_\mu \in \mathbb{R}^{n \times n}$, $A_\sigma \in \mathbb{R}^{n \times n_w}$ and $B_c \in \mathbb{R}^{n \times m}$ encode the known parts of the system dynamics. 
The functions $H_\mu: \mathbb{R}^n \rightarrow \mathbb{R}^m$ and $H_\sigma: \mathbb{R}^n \rightarrow \mathbb{R}^{m \times n_w}$ are continuously differentiable maps describing the drift and diffusion uncertainties, respectively.
Throughout the sequel, we use capital letters (e.g., $X_t$, $U_t$) to denote continuous-time variables, and lowercase letters (e.g., $x_k$, $u_k$) to denote their discrete-time counterparts. 
We refer to $H_\mu(\bar{X}_t)$ and $H_\sigma(\bar{X}_t)$ as epistemic uncertainties, and to $dW_t$ as aleatoric uncertainty.

For completeness, we introduce the following two definitions.
A continuous function $\gamma(x)$ is of class $\pazocal{K}$ if it is strictly increasing in $x$, and satisfies $\gamma(0) = 0$.
A continuous function $\beta(x,t)$ is of class $\pazocal{KL}$ if, for each fixed $t$, $\beta(x,t)$ is of class $\pazocal{K}$, and for each fixed $x$, $\beta(x,t)$ is decreasing with respect to $t$ and satisfies $\lim_{t \rightarrow \infty} \beta(x,t) = 0$.

\subsection{Imitation Learning}

We consider $n$ rollouts of length $L_d$ generated by an expert policy under nominal conditions. For a sufficiently small sampling time $\Delta T$, these rollouts can be interpreted as state-input trajectories of the discretized nominal deterministic dynamics corresponding to \eqref{eq:ito}, given by:
\begin{equation}
    \label{eq:simulator}
    x_{k+1} = A x_k + B u_k, 
\end{equation}
where $k \in \mathbb{N}_0$ denotes the discrete time index, $x_k \in \mathbb{R}^n$ is the state vector, $u_k \in \mathbb{R}^m$ is the input vector, $I_n$ denotes an identity matrix of size $n$,
$A =  I_n + A_\mu \Delta T$ and $B = B_c \Delta T$ define the discrete-time system matrices.
Each rollout is initialized from an independently and identically distributed initial condition $\xi$ drawn from a distribution $\pazocal{D}$ with a compact support $\pazocal{X}$.

Given a policy $\pi(x_k)$, we denote the solution of \eqref{eq:simulator} at time index $k$ under the input $u_k = \pi(x_k)$  starting from $\xi$ by $x_k(\xi,\pi)$. 
We denote the expert policy by $\pi_*(x_k)$, and
the discrepancy between any policy $\pi(x_k)$ and the expert policy $\pi_*(x_k)$ on a trajectory generated by the latter starting from $\xi$ by $\Delta_k^{\pi_*}(\xi;\pi) = \pi(x_k(\xi,\pi_*)) - \pi_*(x_k(\xi,\pi_*))$. 

The main objective of imitation learning (IL) is to learn a policy that induces a closed-loop behavior similar to that of the expert policy, as quantified by the expected imitation gap, $\mathbb{E}_\xi \operatorname{max}_{1 \leq k \leq L_d} \| x_k(\xi,\pi) - x_k(\xi,\pi_*)\|$.
The standard IL approach, known as BC, achieves the aforementioned objective by formulating a supervised learning problem that minimizes the empirical risk as follows:
\begin{equation*}
    \hat{\pi} \in \operatorname{argmin}_{\pi \in \Pi} \sum\nolimits_{i=1}^n h(\{\Delta_k^{\pi_*}(\xi_i;\pi)\}_{k=0}^{L_d}),
\end{equation*}
where $\Pi$ is a preselected policy class, and $h(\cdot)$ is a loss function designed to penalize the discrepancies $\Delta_k^{\pi_*}(\xi_i;\pi)$ along the expert trajectories.

\subsection{Taylor Series Imitation Learning}

Despite the simplicity of BC, the learned policy $\hat{\pi}(x_k)$ can experience distribution shift due to compounding errors. 
To this end, TaSIL replaces the standard BC loss function with
\begin{equation}
    \label{eq:loss}
    l_p^{\pi_*}(\xi,\pi) = \sum\nolimits_{j=0}^{p} \operatorname{max}_{0 \leq k \leq L_d-1} \|\partial_x^j\Delta_k^{\pi_*}(\xi,\pi) \|,
\end{equation}
where $\partial_x^j$ denotes the $j$-th order partial derivative with respect to $x$ and $p$ refers to the highest order derivative considered in the loss function. Note that any surrogate loss that upper bounds the supremum loss~\eqref{eq:loss} can be used. Thus, the empirical risk minimization problem becomes
\begin{equation*}
    \hat{\pi}_p \in \operatorname{argmin}_{\pi \in \Pi} \sum\nolimits_{i=1}^n l_p^{\pi_*} (\xi_i, \pi).
\end{equation*}
The main idea behind TaSIL \cite{pfrommer2022tasil} is to minimize not only the discrepancy between the learned and expert policies, but also between their higher-order derivatives up to order $p$, thereby mitigating policy-induced distribution shifts as follows.

Given a policy $\pi(x_k)$ and an additive input perturbation $\Delta_k$, let the perturbed closed-loop dynamics of the nominal determinstic system~\eqref{eq:simulator} be:
\begin{equation}
    \label{eq:perturbed_system}
    x_{k+1} = A x_k + B (\pi(x_k)+\Delta_k),
\end{equation} 
and its solution at time index $k$ be denoted by $x_k^\pi(\xi,\{\Delta_{s}\}_{s=0}^{k-1})$. 

\begin{assumption}
    \label{ass:ISS}
    The perturbed closed-loop dynamics \eqref{eq:perturbed_system} under the expert policy $\pi_*(x_k)$ is $\eta$-locally $\delta$-input-to-state (ISS) stable, that is, for all $\Delta_k$ satisfying $\operatorname{sup}_{k \in \mathbb{N}_0} \|\Delta_k\| \leq \eta$, there exists a class $\pazocal{KL}$ function $\beta(\cdot,\cdot)$ and a class $\pazocal{K}$ function $\gamma(\cdot)$ such that for all initial conditions $\xi_1$, $\xi_2 \in \pazocal{X}$, perturbation sequences $\{\Delta_k\}_{k \geq 0}$ and $k \in \mathbb{N}$, we have
    \begin{multline}
        \|x_k^{\pi_*}(\xi_1,\{\Delta_s\}_{s=0}^{k-1}) - x_k(\xi_2,\pi_*)\| \leq \beta(\|\xi_1-\xi_2\|,k) \\
        + \gamma(\operatorname{max}_{0 \leq i \leq k-1} \| \Delta_i \|).
    \end{multline}
\end{assumption}

In this setting, a system-theoretic analysis can be sought where the learned policy $\hat{\pi}_p(x_k)$ is viewed as a perturbed version of the ISS expert policy $\pi_*(x_k)$, and the discrepancy $\Delta_k^{\hat{\pi}_p}(\xi,\hat{\pi}_p) = \hat{\pi}_p(x_k(\xi,\hat{\pi}_p)) - \pi_*(x_k(\xi,\hat{\pi}_p))$ acts as the perturbation.
In this case, the imitation gap defined as $\Gamma_T(\xi,\hat{\pi}_p) = \operatorname{max}_{1 \leq k \leq L_d} \| x_k(\xi,\hat{\pi}_p) - x_k(\xi,\pi_*)\|$ is expressed as $\Gamma_T(\xi,\hat{\pi}_p) = \operatorname{max}_{1 \leq k \leq L_d} \| x_k^{\pi_*}(\xi,\{\Delta_s^{\hat{\pi}_p}(\xi,\hat{\pi}_p)\}_{s=0}^{k-1}) - x_k(\xi,\pi_*)\|$.
Under mild regularity assumptions (see \cite{pfrommer2022tasil} for more details), and for sufficiently large number of sampled trajectories, there exists a constant $\rho_{T} \geq 0$ such that the imitation gap evaluated on $\xi \sim \pazocal{D}$ satisfies $\Gamma_T(\xi,\hat{\pi}_p) \leq \rho_{T}$ with high probability.

Since TaSIL aims to reduce policy-induced distribution shift, it performs well in mimicking the expert behavior on the nominal system used to collect the demonstrations.
However, it does not account for uncertainty-induced distribution shifts arising from epistemic and aleatoric uncertainties in the true system. 
Consequently, its performance is only optimized under nominal conditions. 
Furthermore, it does not incorporate safety requirements during training.
In the sequel, we address these limitations by developing a distributionally robust and safe TaSIL framework.

\section{Proposed Approach}

In this section, we develop a distributionally robust and safe TaSIL framework that, 
(I) achieves robustness to uncertainty-induced distribution shifts, 
(II) optimizes performance under distributional uncertainty, and
(III) accounts for safety requirements during training.

\subsection{Robustness}

We employ $\pazocal{L}_1$-DRAC to achieve robustness of the true system~\eqref{eq:ito} in the presence of uncertainties. 
To this end, we define the nominal stochastic system:
\begin{equation}
    \label{eqn:NominalSystem}
    dX^*_{t} 
    =
    \left[ 
    A_\mu X^*_{t} + B_c U^*_{t} 
    \right] dt
    + A_\sigma  dW^*_{t}, 
\end{equation}
where $X^* \in \mathbb{R}^n$ and $U^* \in \mathbb{R}^m$ represent the nominal state and input vectors, respectively, and $W^*_{t} \in \mathbb{R}^{n_w}$ is another Brownian motion defined on $\left(\Omega,\pazocal{F},\mathbb{P}\right)$, and is independent of $W_{t}$. 
Note that the nominal system~\eqref{eqn:NominalSystem} corresponds to the true system~\eqref{eq:ito} with $H_\mu$ and $H_\sigma$ in~\eqref{eq:ito} set to zero. We also define $B_u \in \mathbb{R}^{n \times (n-m)}$ such that $\bar{B}=[B_c \ B_u]$ has full rank.

By setting \( U_t^\star \) as the control input obtained via a zero-order hold of the proposed TaSIL policy \( \hat{\pi}_p^s \) (derived in subsequent sections and inspired by the standard TaSIL policy \( \hat{\pi}_p \)), we consider the input to the true system~\eqref{eq:ito} of the form:
\begin{equation}\label{eq:input}
    \bar{U}_{t} = U^*_{t} + U_{\pazocal{L}_1,t},
\end{equation}
where $U_{\pazocal{L}_1,t} \in \mathbb{R}^m$ is the $\pazocal{L}_1$-DRAC input \cite{gahlawat2025mathcal}. This input is defined as the output of the low pass filter:
\begin{align*}
    U_{\pazocal{L}_1,t}
    =  
    - \omega \int_0^t e^{-\omega(t-\nu)} \hat{\Lambda}_{m,\nu} d\nu,  
\end{align*}
where $\omega \in \mathbb{R}_{>0}$ is the filter bandwidth, $0_{p \times q}$ denotes a zero matrix of size $p \times q$, and $\hat{\Lambda}_{m,t} = \begin{bmatrix}I_m & 0_{m\times (n-m)}  \end{bmatrix} \bar{B}^{-1} \hat{\Lambda}_t$. The variable $\hat{\Lambda}_t \in \mathbb{R}^m$ is obtained via the adaptation law:
\begin{multline*}
    \hat{\Lambda}_t 
    =  
    0_n {1}_{[0,T_s)}
    +
    \textstyle \lambda_s (1 - e^{\lambda_s T_s})^{-1} e^{\lambda_s T_s}
    \sum_{i=1}^{\lfloor \frac{t}{T_s} \rfloor}    
    \tilde{X}_{iT_s}
    {1}_{[iT_s,(i+1)T_s)},
\end{multline*} 
with $\tilde{X}_{iT_s} = \hat{X}_{iT_s} - \bar{X}_{iT_s}$, $\lambda_s \in \mathbb{R}_{>0}$, $T_s \in \mathbb{R}_{>0}$, and $1_{[t_1,t_2)}$ being $1$ in the time interval $[t_1,t_2)$ and 0 otherwise.
For an initial condition $X_0$, $\hat{X}_{t}$ is the solution of the process predictor:
\begin{equation*}
    \hat{X}_{t} 
    =
    X_0
    \textstyle +
     \int_0^t 
        \left(-\lambda_s {I}_n \tilde{X}_{\nu}+ A_\mu \bar{X}_\nu
        +
        B_c (U_\nu^* + U_{\pazocal{L}_1,\nu}) + \hat{\Lambda}_\nu\right) d\nu,
\end{equation*} 
where $\tilde{X}_{t} = \hat{X}_{t} - \bar{X}_{t}$.
We refer to $\{\omega,T_s,\lambda_s\}$ as the  $\pazocal{L}_1$-DRAC parameters, and denote the proposed $\pazocal{L}_1$-DRAC-augmented TaSIL policy by $\hat{\pi}_{p}^{l}$.
Note that the $\pazocal{L}_1$-DRAC augmentation is independent of the choice of the order $p$.

Notice that, given a policy $\pi$ and its corresponding $\pazocal{L}_1$-DRAC-augmented policy $\pi^l$, it is shown in~\cite{gahlawat2025mathcal} that there exists a scalar $\rho \geq 0$ such that
$\mathbb{W}_{2}(\mathbb{X}_t(\xi,\pi^l),\mathbb{X}_t^*(\xi,\pi)) \leq \rho$ for all $t \in [0,T_h]$,
under mild regularity assumptions on the nominal dynamics and the uncertain terms in the true system, and with an appropriate choice of the $\pazocal{L}_1$-DRAC parameters.
Here, $T_h$ denotes the time horizon, $\mathbb{W}_{2}$ is the 2-Wasserstein distance, $\mathbb{X}_{t}(\xi,\pi^l)$ denotes the distribution of the true system~\eqref{eq:ito} under $\pi^l$ starting from $\xi$, and $\mathbb{X}_{t}^*(\xi,\pi)$ denotes the distribution of the nominal system~\eqref{eqn:NominalSystem} under $\pi$ starting from $\xi$.
The specific case in which $\pi$ is the standard TaSIL policy is addressed in~\cite{gahlawat2026distributionally}.
In the interest of space, we refer the reader to~\cite{gahlawat2026distributionally,gahlawat2025mathcal} for further details on the considered assumptions and their implications.

In the following subsections, we show how $\pazocal{L}_1$-DRAC can inspire the offline training of TaSIL, by optimizing performance under distributional uncertainty while incorporating safety requirements.

\subsection{Performance}

We propose a TaSIL formulation that optimizes performance under distributional uncertainty. To this end, we use the Euler-Marumaya scheme to reach a discretized version of the true system~\eqref{eq:ito} given by
\begin{multline}
    \label{eq:true}
    \bar{x}_{k+1} = A \bar{x}_k + B ( \bar{u}_k + H_\mu(\bar{x}_{k}) )
    + ( A_\sigma + B_c H_{\sigma}(\bar{x}_k) ) \Delta w_k, 
\end{multline}
where $\Delta w_k \overset{\text{i.i.d.}}{\sim} \pazocal{N}(0_{n_w},\Delta T\mathbb{I}_{n_w})$, and a discretized version of the nominal system~\eqref{eqn:NominalSystem} given by
\begin{align}
    \label{eq:nominal}
    x_{k+1}^* = A x_k^* + B u_k^* + A_\sigma \Delta w_k^*, 
\end{align}
where $\Delta w_k^* \overset{\text{i.i.d.}}{\sim} \pazocal{N}(0_{n_w},\Delta T\mathbb{I}_{n_w})$.

Assume that there exists a scalar $\rho \geq 0$ such that $\mathbb{W}_{2} (\mathbb{X}_t(\xi,\hat{\pi}_p^l),\mathbb{X}_t^*(\xi,\hat{\pi}_p^s)) \leq \rho \text{ for all } t \in [0,T_h]$.
It should be noted that this bound may not hold exactly for the corresponding discrete-time systems.
As a result, it is necessary to consider the errors introduced by temporal discretization.
As mentioned in \cite{gahlawat2025wasserstein}, one can inflate the constant $\rho$ to account for the discretization error. 
For simplicity, we maintain the constant $\rho$ and assume that the Wasserstein distance between $\mathbb{X}_k(\xi,\hat{\pi}_p^l)$ and $\mathbb{X}_k^*(\xi,\hat{\pi}_p^s)$ satisfies
$ \mathbb{W}_{2} (\mathbb{X}_k(\xi,\hat{\pi}_p^l),\mathbb{X}_k^*(\xi,\hat{\pi}_p^s)) \leq \rho $
where, for an initial condition $\xi$, $\mathbb{X}_k(\xi,\hat{\pi}_p^l)$ and $\mathbb{X}_k^*(\xi,\hat{\pi}_p^s)$ are the distributions of $\bar{x}_k$ under $\hat{\pi}_p^l$ and $x_k^*$ under $\hat{\pi}_p^s$, respectively.

\begin{prop}
    \label{prop}
    Let $f_{cl}(x)=Ax+B\hat{\pi}_p^s(x)$ be Lipschitz continuous with constant $L$.
    Consider the systems in \eqref{eq:simulator} and \eqref{eq:nominal}, both initialized at $\xi$ and evolving under $\hat{\pi}_p^s$. Let $\delta_{x_{k}(\xi,\hat{\pi}_p^s)}$ be the Dirac distribution of the nominal deterministic state in~\eqref{eq:simulator} at time $k$. 
    Then, there exists a constant $\rho_a$ such that $\mathbb{W}_{2} (\mathbb{X}_k^*(\xi,\hat{\pi}_p^s),\delta_{x_{k}(\xi,\hat{\pi}_p^s)}) \leq \rho_a$ for all $k \in \{0,\hdots,\lfloor T_h/\Delta T \rfloor\}$.
\end{prop}
\begin{proof}
    Define the error $e_k=x_k^*-x_k$ at time $k$.
    The error dynamics become $e_{k+1}=f_{cl}(x_k^*)-f_{cl}(x_k)+ A_\sigma \Delta w_k^*$.
    Since $f_{cl}$ is Lipschitz with constant $L$, we have $\|e_{k+1}\| \leq L\|e_k\|+\|A_\sigma\Delta w_k^*\|$. 
    Since $\Delta w_k^*$ is independent of $e_k$ and has zero mean, we reach $\mathbb{E}[\|e_{k+1}\|^2] \leq L^2\mathbb{E}[\|e_k\|^2]+\operatorname{tr}(A_\sigma A_\sigma^\top)\Delta T$.
    With $e_0=0$, this recursion implies that $e_{k}$ is finite for any finite $k$.
    Due to the Dirac distribution $\delta_{x_{k}(\xi,\hat{\pi}_p^s)}$, the 2 Wasserstein distance $\mathbb{W}_{2} (\mathbb{X}_k^*(\xi,\hat{\pi}_p^s),\delta_{x_{k}(\xi,\hat{\pi}_p^s)})^2 = \mathbb{E}[\|e_{k}\|^2] \leq \rho_a$ where $\rho_a = \max_{k \in \{0,\hdots,\lfloor T_h/\Delta T\rfloor\}}\mathbb{E}[\|e_{k}\|^2]$.
\end{proof}

The Lipschitz condition in Proposition~\ref{prop} is mild, as many neural networks are Lipschitz under standard weight and activation assumptions.
By defining $\rho_L = \rho + \rho_a$ and using the triangle inequality, it is deduced that $\mathbb{W}_{2} (\mathbb{X}_k(\xi,\hat{\pi}_p^l),\delta_{x_{k}(\xi,\hat{\pi}_p^s)}) \leq \rho_L$.
To optimize performance in the presence of distributional uncertainty,
we propose the risk minimization problem:
\begin{equation*}
    \hat{\pi}_p^r \in \operatorname{argmin}_{\pi \in \Pi}
    \mathbb{E}_{\xi \sim \pazocal{D}}\left[ l_{p,\pazocal{L}_1}^{\pi_*,r} (\xi, \pi)\right],
\end{equation*}
where the loss function $l_{p,\pazocal{L}_1}^{\pi_*,r} (\xi, \pi)$ is given by
\begin{align}
    \label{eq:Rloss}
    l_{p,\pazocal{L}_1}^{\pi_*,r} (\xi, \pi) 
    &= 
    \sum\nolimits_{j=0}^{p} 
    \max_{0 \leq k \leq L_d-1} 
    \sup_{\nu_k \in \mathbb{B}_{\rho_L}(\delta_{x_k(\xi,\pi_*)})} 
    Y_k^j(\xi), \\
    Y_k^j(\xi)
    &= 
    \nonumber
    \mathbb{E}_{y_k \sim \nu_k} 
    \| 
    \partial_{y_k}^j \pi(y_k) 
    - 
    \partial_{y_k}^j \pi_*(y_k) 
    \|,
\end{align}
and $\mathbb{B}_{\rho_L}\left(\delta_{x_k(\xi,\pi_*)}\right) := \left\{ \mathbb{Q} \ | \ \mathbb{W}_2\left(\mathbb{Q}, \delta_{x_k(\xi,\pi_*)}\right) \le \rho_L \right\}$.
Note that we construct the ball $\mathbb{B}_{\rho_L}$ around the expert trajectory $x_k(\xi,\pi_*)$, since it is known a priori, whereas $x_k(\xi,\hat{\pi}_p^s)$ is not.
The accuracy of this approximation depends on how closely the learned policy matches the expert; in general, $\rho_L$ can be enlarged to account for the approximation error.

The intuition behind \eqref{eq:Rloss} is as follows. 
Instead of aligning the learned policy (and its derivatives) with the expert policy (and its derivatives) only along the expert trajectory, we consider a ball of distributions centered at the Dirac distribution of the expert trajectory. 
Trajectories are then sampled from the worst-case distribution within this ball, and the learned policy (and its derivatives) are aligned along them.
This approach ensures that the TaSIL policy aligns as closely as possible with the expert policy under uncertainty.

The main limitation of this loss function is that the expert policy and its derivatives (up to order $p$) are not available along the sampled trajectories. However, they are available along the expert trajectories, in whose neighborhood the sampled trajectories lie. Thus, we approximate the expert policy locally using a Taylor expansion around the expert trajectory. In particular, for a state $y_k$ in the vicinity of $x_k(\xi,\pi_*)$, we consider the surrogate,
$
\pi_*(y_k) \approx \pi_*(x_k(\xi,\pi_*)) + J_{\pi_*}(x_k(\xi,\pi_*))(y_k - x_k(\xi,\pi_*)) + \cdots,
$
up to order $p$,
where $J_{\pi_*}(x_k^{\pi_*}(\xi))$ denote the Jacobian of the expert policy evaluated at $x_k^{\pi_*}(\xi)$, and is available by assumption.

\begin{theorem}
    If no epistemic or aleatoric uncertainty is present, then the distributionally robust TaSIL formulation in~\eqref{eq:Rloss} coincides with the standard TaSIL formulation in~\eqref{eq:loss}.
\end{theorem}

\begin{proof}
    In the absence of epistemic and aleatoric uncertainty, the system in~\eqref{eq:true} boils down to the system in~\eqref{eq:simulator}.
    In this case, the distribution $\mathbb{X}_k(\xi,\pi_*)$ of the true system matches the Dirac distribution $\delta_{x_k(\xi,\pi_*)}$ of the nominal deterministic system.
    Thus, $\mathbb{W}_2(\mathbb{X}_k(\xi,\pi_*),\delta_{x_k(\xi,\pi_*)})=0$, and hence, $\mathbb{B}_{\rho_L}\left(\delta_{x_k(\xi,\pi_*)}\right) = \{ \delta_{x_k(\xi,\pi_*)} \}$.
    To this end, $\nu_k=\delta_{x_k(\xi,\pi_*)}$ in~\eqref{eq:Rloss}, the inner supremum problem reduces to $\|\partial_{x_k(\xi,\pi_*)}^j \pi(x_k(\xi,\pi_*)) 
    - 
    \partial_{x_k(\xi,\pi_*)}^j \pi_*(x_k(\xi,\pi_*))\|$, and the formulations in~\eqref{eq:loss} and \eqref{eq:Rloss} are equivalent.
\end{proof}

Following \cite{sinha2017certifying},
\color{black}
one can replace the inner supremum problem by its relaxed Lagrangian in~\eqref{eq:Rloss} and use a robust surrogate of the resulting problem for training.
In this case, gradient ascent can be used to solve the supremum problem.

\subsection{Safety}

\begin{figure*}
    \label{fig:tracking}
    \centering
    \includegraphics[trim=0pt 0pt 0pt 0pt, clip, width=0.8\textwidth]{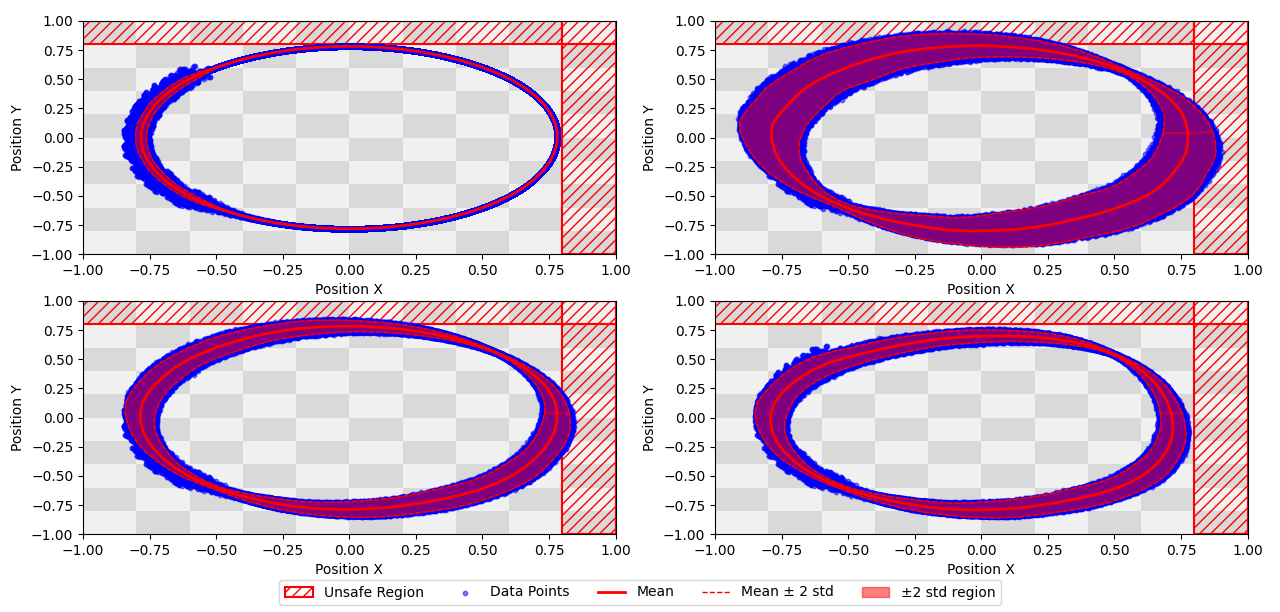}
    \caption{Monte Carlo simulations of the task under different controllers and operating conditions. \textbf{(Top-Left)} Standard TaSIL policy in the absence of uncertainties. \textbf{(Top-Right)} Standard TaSIL policy in the presence of uncertainties. \textbf{(Bottom-Left)} Standard TaSIL policy augmented with $\pazocal{L}_1$-DRAC in the presence of uncertainties. \textbf{(Bottom-Right)} Distributionally robust and safe TaSIL policy augmented with $\pazocal{L}_1$-DRAC in the presence of uncertainties.
    }
    \label{fig:TaSIL}
\end{figure*}

While the expert can account for safety requirements when collecting demonstrations under nominal conditions, TaSIL does not explicitly take these requirements into account. Here, we consider a safe region whose admissible set is
$\pazocal{X} = \{x \in \mathbb{R}^n : c_j^\top x \leq d_j, \ \forall j \in \{1,\hdots,n_c\}\},$
where $n_c$ is the number of constraints with 
$c_j \in \mathbb{R}^{n}$ and $d_j \in \mathbb{R}$ 
for all $j \in \{1,\hdots,n_c\}$.
Starting from an expert data point $x_k(\xi,\pi_*)$, we define the next nominal and true states under a policy $\pi$, respectively, as $z_k(\xi) = A x_k(\xi,\pi_*) + B \pi(x_k(\xi,\pi_*))$ and $\bar{z}_k(\xi) = A x_k(\xi,\pi_*) + B ( \pi(x_k(\xi,\pi_*)) + H_\mu(x_k(\xi,\pi_*)) ) + ( A_\sigma + B_c H_{\sigma}(x_k(\xi,\pi_*)) ) \Delta w_k$ following~\eqref{eq:simulator} and~\eqref{eq:true}.
Note that the $\bar{z}_k(\xi)$ has a distribution $\mathbb{X}_1(x_k(\xi,\pi_*),\pi)$.

Due to the stochastic nature of $\bar{z}_k(\xi)$, we consider the chance constraints,
$\mathbb{P}(\bar{z}_k(\xi) \in \pazocal{X}) \geq 1-\delta_s$, 
where $\delta_s \in (0,1)$ is a user specified risk parameter.
Owing to their convexity and ability to distinguish tail events, we approximate the chance constraints using the Conditional-Value-at-Risk (CVaR) constraints as follows:
\begin{equation}
    \label{eq:cvar}
    \operatorname{CVaR}_{1-\delta_s}^{\bar{z}_k(\xi) \sim \mathbb{X}_1(x_k(\xi,\pi_*),\pi)}\left(\max_{j \in  \{1,\hdots,n_c\}} c_j^\top \bar{z}_k(\xi) - d_j\right) \leq 0.
\end{equation}
The main challenge in \eqref{eq:cvar} is that $\mathbb{X}_1(x_k(\xi,\pi_*),\pi)$ is unknown due to the existing uncertainties.
Note, however, that $\mathbb{W}_{2} (\mathbb{X}_1(x_k(\xi,\pi_*),\pi),\delta_{z_k(\xi)}) \leq \rho_L$ from Proposition~\ref{prop}.
Thus, we can replace~\eqref{eq:cvar} with $Z_k(\xi) \leq 0$ where
\begin{equation}
    \label{eq:cvar_sup}
    Z_k(\xi) = \sup_{\nu_k \in \mathbb{B}_{\rho_L}(\delta_{z_k(\xi)})} \operatorname{CVaR}_{1-\delta_s}^{\bar{z}_k(\xi) \sim \nu_k}\left(\max_{j \in  \{1,\hdots,n_c\}} c_j^\top \bar{z}_k(\xi) - d_j\right).
\end{equation}
Inspired by~\eqref{eq:cvar_sup}, we aim to add a safety component to the TaSIL loss function to promote safety.
To this end, we propose the risk minimization problem:
\begin{equation*}
    \hat{\pi}_p^s \in \operatorname{argmin}_{\pi \in \Pi}
    \mathbb{E}_{\xi \sim \pazocal{D}}\left[ l_{p,\pazocal{L}_1}^{\pi_*,s} (\xi, \pi)\right],
\end{equation*}
where the loss function $l_{p,\pazocal{L}_1}^{\pi_*,s} (\xi, \pi)$ combines the distributionally robust term \eqref{eq:Rloss} with an additional physics-informed component, and is defined as
\begin{align}
    l_{p,\pazocal{L}_1}^{\pi_*,s} (\xi, \pi) 
    =
    l_{p,\pazocal{L}_1}^{\pi_*,r} (\xi, \pi)
    +
    \lambda
    \sum\nolimits_{k=0}^{L_d-1}
    \operatorname{ReLU} \left( Z_k(\xi)
    \right), 
\end{align}
where $\lambda \in \mathbb{R}_+$ and $\operatorname{ReLU}$ denotes the rectified linear unit.

The intuition behind this loss function is as follows. 
When the CVaR constraint is satisfied (i.e., its value is negative), the ReLU term evaluates to zero and the loss reduces to the distributionally robust TaSIL objective, which encourages the learned policy to mimic the expert policy under uncertainty. 
However, if the CVaR constraint is violated, typically because the expert demonstrations were collected under conditions that differ from those at deployment, the ReLU becomes positive and adds a penalty. 
This discourages the policy from blindly imitating the expert and instead guides it to remain as close as possible while respecting the constraint. 
Note that ReLU is only one possible choice for penalization; smooth alternatives include  ELU and GeLU.

Since the safe region is represented using linear constraints, 
$Z_k(\xi) \leq 0$ can be expressed, following \cite{gahlawat2025wasserstein}, as
\begin{equation}
    Z_k^j(\xi) = c_j^\top z_k(\xi) - d_j 
    + \frac{\|c_j\|_2}{\sqrt{\delta_s/n_c}}\rho_L \leq 0, \forall j \in \{1,\hdots,n_c\}.
\end{equation}
In this case, 
\color{black}
the loss function $l_{p,\pazocal{L}_1}^{\pi_*,s} (\xi, \pi)$ is given by
\begin{multline}
    \label{eq:Sloss}
    l_{p,\pazocal{L}_1}^{\pi_*,s} (\xi, \pi) =
    l_{p,\pazocal{L}_1}^{\pi_*,r} (\xi, \pi)
    +
    \lambda
    \sum_{j=1}^{n_c}
    \sum_{k=0}^{L_d-1}
    \operatorname{ReLU} \left(
    Z_k^j(\xi)
    \right).
\end{multline}
While the safety term considers one-step-ahead states initialized from expert data points, it naturally extends to multi-step-ahead predictions. 
Moreover, instead of relying on expert trajectories, one may use the trajectories sampled in~\eqref{eq:Rloss} via gradient ascent.
Finally, the proposed approach employs the bound $\rho_L = \rho + \rho_a$ for learning the TaSIL policy, as in~\eqref{eq:Sloss}. The computation of the $\pazocal{L}_1$-DRAC bound, however, depend on the nominal dynamics under the learned policy, leading to a coupled co-design problem. 
This challenge is further compounded by the requirement of a contractive nominal policy for the computation of the $\pazocal{L}_1$-DRAC bound. While this requirement is discussed in~\cite{gahlawat2026distributionally}, this aspect would benefit from a more thorough analysis.
Tools from the co-design theory \cite{censi2015mathematical} provide an interesting framework to jointly design TaSIL and $\pazocal{L}_1$-DRAC; this direction is left for future work. In the present work, $\rho_L$ is chosen empirically.

\section{Simulation Results}

In this section, we demonstrate the effectiveness of the proposed framework on a UAV case study, in which the goal is to emulate an expert planner 
in an uncertain environment while accounting for safety requirements.

We consider a UAV model
where the state vector is given by $x = [p^\top, v^\top, \eta^\top, \omega^\top]^\top$, 
comprising the position $p = [p_x, p_y, p_z]^\top$, linear velocity $v = [v_x, v_y, v_z]^\top$, Euler angles $\eta = [\phi, \theta, \psi]^\top$, and body angular rates $\omega = [p, q, r]^\top$ of the vehicle. 
The input vector is given by  $u = [u_1, u_2, u_3, u_4]^\top$, 
where $u_1$ denotes the total thrust and $u_2, u_3, u_4$ denote the roll, pitch, and yaw moments, respectively.
We assume a low-level controller, in the form of an LQR law given by $u = -Kx + Ld$, is used to stabilize the UAV, where $d = [x_d, y_d, z_d, \psi_d]^\top$ denotes the desired trajectory for the UAV position $p$ and heading $\psi$. The matrices $K$ and $L$ are control gains of appropriate dimensions.
At time index $k$, the desired trajectory is updated as $d_k = d_{k-1} + \Delta d_{k-1}$.
Our objective is to learn a policy that imitates the expert planner generating the trajectory increments $\Delta d_{k-1}$.

Let $m$, $J_x$, $J_y$, $J_z$, and $g$ denote the UAV's mass, moments of inertia, and gravitational acceleration, respectively. The UAV model can then be written in the form~\eqref{eq:ito} with $X_t \in \mathbb{R}^{12}$ representing the UAV state vector, $U_t \in \mathbb{R}^{4}$ denoting the UAV desired trajectory, $A_\mu = A_q - B_q K$, $B_c = B_q L$, where
$B_q = \left[ \frac{1}{m}e_6 \ \frac{1}{J_x}e_{10} \ \frac{1}{J_y}e_{11} \ \frac{1}{J_z}e_{12} \right]$,
$A_\sigma = \left[ \sigma_v e_4 \ \sigma_v e_5 \ \sigma_v e_6 \ \sigma_\eta e_{10} \ \sigma_\eta e_{11} \ \sigma_\eta e_{12} \right]$, and
\begin{equation*}
    A_q = 
    \begin{bmatrix}
        0_{9 \times 3} & \operatorname{diag}(I_3,A_{23},0,I_3) \\
        0_3 & 0_{3 \times 9}
    \end{bmatrix}\hspace{-0.1cm}, \ 
    A_{23} =
    \begin{bmatrix}
        s_{\psi_0}g & c_{\psi_0}g \\ 
        -c_{\psi_0}g & s_{\psi_0}g
    \end{bmatrix}\hspace{-0.1cm}.
\end{equation*}
The parameter $\psi_0$ denotes the heading around which the UAV operates, 
$e_i \in \mathbb{R}^{12}$ denote unit vectors whose $i$-th element is 1 for all $i \in \{1,\hdots,12\}$, 
$\sigma_v \in \mathbb{R}_+$,
$\sigma_\eta \in \mathbb{R}_+$,
$s_\theta$ and $c_\theta$ denote $\sin(\theta)$ and $\cos(\theta)$, respectively.

In numerical simulations, the model parameters are set as $m = 1\mathrm{kg}$, $g = 9.81\mathrm{m/s^2}$, $J_x = J_y = 0.01\mathrm{kg,m^2}$, $J_z = 0.02\mathrm{kg,m^2}$, and $\psi_0 = 0$, while $\sigma_v = 0.1$ and $\sigma_\eta = 0.1$.
The weighting matrices $Q = \operatorname{diag}(10,10,100,0,0,0,0,0,10,0,0,0)$ and $R = \operatorname{diag}(1,0.1,0.1,0.1)$ are used to compute the LQR gains $K$ and $L$.
The expert planner generates trajectories that drive the UAV to follow a circular horizontal trajectory of radius $0.8$ in a counterclockwise direction, from any initial condition in the vicinity of the trajectory. 
The safe region is defined as $\{(x,y) \in \mathbb{R}^2 \mid x \le 0.8,\; y \le 0.8\}$.

Following the notation in~\eqref{eq:ito}, we consider the epistemic uncertainty $H_\mu(x)$ that affects the UAV through the same channels as $U_t$, i.e., the desired trajectory. This disturbance can equivalently be viewed as disturbing forces and moments $H(x) = L H_\mu(x)$ acting through the same channels as the control input vector $u$.
The epistemic uncertainty $H_{\mu}(x)$ is modeled as a sum of sinusoids with randomly generated amplitudes and frequencies. The frequencies are uniformly sampled from $[3, 4]\mathrm{rad/s}$. The amplitudes in the $d_x$, $d_y$, $d_z$ and $d_\psi$ channels are uniformly sampled from $[0.02,0.025]$, $[0.02,0.025]$, $[0.04,0.05]$, and $[0.04,0.05]$, respectively.
In this case study, we set $H_\sigma(x)$ to zero.

We consider learned policies comprising two hidden layers, each of which has 32 neurons and uses GELU as activation functions. 
In all simulations, we use TaSIL of order $p=1$.
The policy is trained over 1000 iterations using $10$ expert trajectories, each of which has $150$ data points and starts from an initial condition 
sampled from a uniform distribution in the vicinity of the desired trajectory.
The $\pazocal{L}_1$-DRAC parameters are given by $\omega=5$, $T_s=0.01$, and $\lambda=5$.

We analyze the system under the following policies and operating conditions: standard TaSIL $\hat{\pi}_p$ in the absence of uncertainties, standard TaSIL $\hat{\pi}_p$ in the presence of uncertainties, standard TaSIL $\hat{\pi}_p$ augmented with $\pazocal{L}_1$-DRAC in the presence of uncertainties, and distributionally robust and safe TaSIL $\hat{\pi}_p^s$ augmented with $\pazocal{L}_1$-DRAC in the presence of uncertainties.
We simulate~\eqref{eq:ito} using the Euler-Maruyama scheme with a sampling time $\Delta T = 0.01$. The Brownian motion increments are modeled as i.i.d. Gaussian random variables with zero mean and variance $\Delta T$.
The planner operates at a slower time scale, with a $0.1$ sampling time.

Figure \ref{fig:TaSIL} shows Monte Carlo simulations of the system behavior under varying initial conditions and epistemic uncertainties, while performing the desired tasks across the four aforementioned cases. 
The initial conditions are sampled near the left quarter of the trajectory, whereas the epistemic uncertainties are constructed as sums of sinusoids, as described earlier.
As shown in the top-left panel, the standard TaSIL policy successfully mimics the expert behavior, following a circular trajectory of radius $0.8$, in uncertainty-free environments. 
However, in the presence of uncertainty (top-right panel), the policy deviates from the expert. 
With the $\pazocal{L}_1$-DRAC augmentation (bottom-left panel), the policy remains closer to the desired trajectory under uncertainty. 
Notably, in all of the above cases, the UAV fails to avoid the unsafe region.
In contrast, the distributionally robust and safe TaSIL (bottom-right panel) closely follows the expert (e.g., along the bottom-left portion of the trajectory) while intentionally deviating near unsafe regions (e.g., along the top-right portion) to ensure safety.

\section{CONCLUSIONS}

We propose a novel distributionally robust and safe imitation learning framework. 
The developed framework is designed to be robust to both policy-induced and uncertainty-induced distribution shifts, optimize performance under distributional uncertainty, and account for safety requirements during training. 
We validate the proposed framework through extensive simulations of a UAV flying in uncertain environments while avoiding unsafe regions.

\addtolength{\textheight}{-12cm}   








\bibliographystyle{ieeetr}
\bibliography{References}

\end{document}